\documentclass{article}

\usepackage[utf8]{inputenc}
\usepackage{float}
\usepackage{amssymb} %deck symbols
\usepackage{amsmath}
\usepackage{amsthm} %theorem environments
\usepackage[ruled]{algorithm2e}
\usepackage{algorithmic}
\usepackage[letterpaper, margin=1.4in]{geometry}
\usepackage{xcolor} %red color for hearts and diamonds
\usepackage{lipsum} %fill-up text
\usepackage{caption}
\usepackage{subcaption}
\newtheorem{theorem}{Theorem}

\theoremstyle{definition}
\newtheorem{definition}{Definition}
\usepackage{tikz}
\usepackage{authblk}
\usepackage{titlefoot}

\theoremstyle{definition}

\title{Learning Linear Non-Gaussian Graphical Models with Multidirected Edges}
\author{Yiheng Liu}
\author{Elina Robeva}
\author{Huanqing Wang}
\affil[]{\normalsize\textit{The University of British Columbia}}
\date{}

\begin{document}

\maketitle

\begin{abstract}
In this paper we propose a new method to learn the underlying acyclic mixed graph of a linear non-Gaussian structural equation model given observational data. We build on an algorithm proposed by Wang and Drton~\cite{BANG}, and we show that one can augment the hidden variable structure of the recovered model by learning {\em multidirected edges} rather than only directed and bidirected ones. Multidirected edges appear when more than two of the observed variables have a hidden common cause. We detect the presence of such hidden causes by looking at higher order cumulants and exploiting the multi-trek rule~\cite{multiTrek}. Our method recovers the  correct structure when the underlying graph is a bow-free acyclic mixed graph with potential multi-directed edges.
%is applicable to the case of bow-free acyclic mixed graphs. %In particular, the higher order cumulants of a linear structural equation model correspond to the presence of {\em multi-treks}~\cite{multiTrek} in the graph, which is what allows us to discover multidirected edges.
\end{abstract}

\unmarkedfntext{\noindent \hspace{-0.33cm} Keywords: Graphical models, Linear Structural Equation Models, Non-Gaussian variables, multi-treks, high-order cumulants

{{MSC2020}} Subject Classification: {62H22, 62R01, 62J99}}

\section{Introduction}
Building on the theory of causal discovery from observational data, we propose a method to learn linear non-Gaussian structural equation models with hidden variables. Importantly, we allow each hidden variable to be a parent of multiple of the observed variables, and we represent this graphically via {\em multi-directed} edges. Therefore, given observational data, we seek to find an acyclic mixed graph whose vertices correspond to the observed variables, and which has directed and multi-directed edges.

\begin{figure}[H]
     \centering
     \begin{subfigure}[b]{0.31\textwidth}
         \centering
         \includegraphics[width=1.1\textwidth]{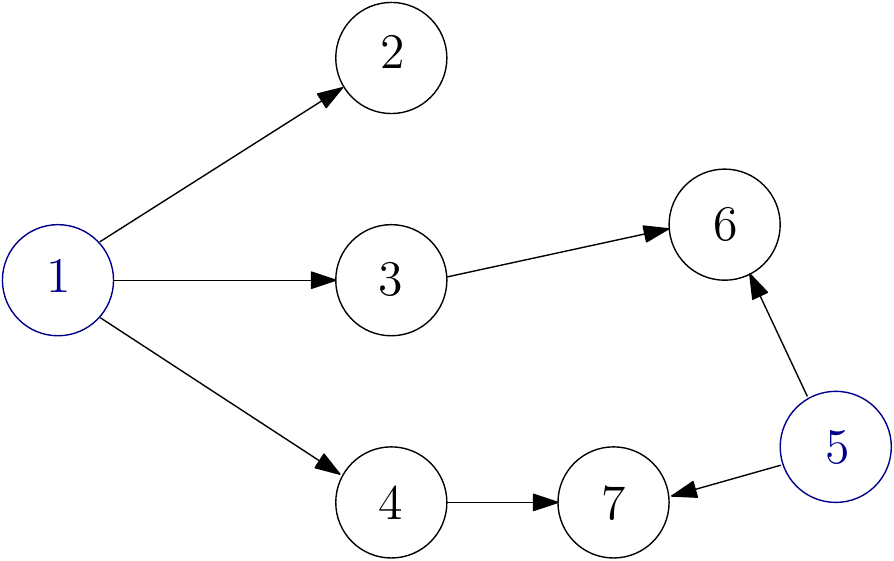}
         \caption{Example DAG.}
         \label{fig:1a}
     \end{subfigure}
     \hfill
     \begin{subfigure}[b]{0.31\textwidth}
         \centering
         \includegraphics[width=0.82\textwidth]{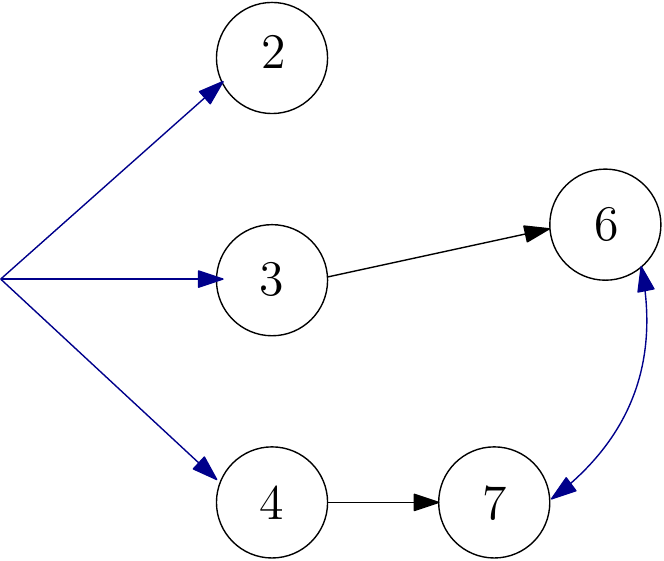}
         \caption{Mixed graph.}
         \label{fig:1b}
     \end{subfigure}
     \hfill
     \begin{subfigure}[b]{0.31\textwidth}
         \centering
         \includegraphics[width=0.65\textwidth]{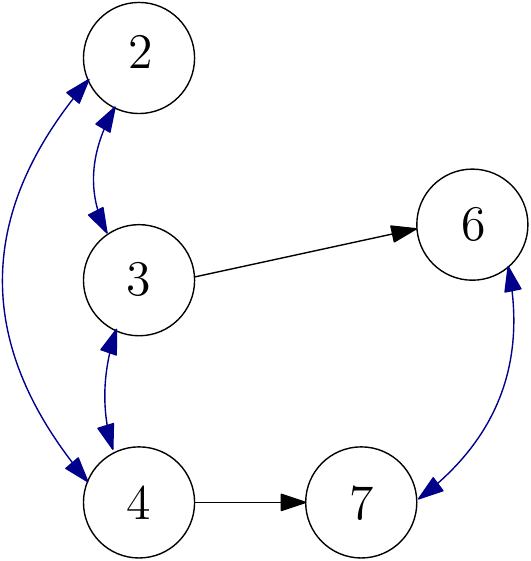}
         \caption{Mixed graph.}
         \label{fig:1c}
     \end{subfigure}
        \caption{The graph in Figure~\ref{fig:1a} is directed acyclic. After marginalization of vertices 1 and 5, the resulting mixed graph with directed and bidirected edges only is that in Figure~\ref{fig:1c}. If we allow multidirected edges, we can capture the hidden variable structure better via the graph in Figure~\ref{fig:1b}, which also has a 3-directed edge.}
        \label{fig:1}
\end{figure}
 Consider an acylic mixed graph $G = (V, \mathcal D, \mathcal H)$, where $V = \{1, \ldots, p\}$ is the set of vertices, $\mathcal D$ is the set of directed edges, and $\mathcal H$ is the set of multi-directed edges, i.e., $\mathcal H$ consists of tuples of vertices such that the vertices in each tuple have a hidden common cause. The graph in Figure~\ref{fig:1c} has only directed and bidirected edges, while the one in Figure~\ref{fig:1b} also has a 3-directed edge. Both of these mixed graphs represent the directed acyclic graph in Figure~\ref{fig:1a} with the variables 1 and 5 unobserved. Representing hidden variables via such edges is quite commonly used for causal discovery~\cite{handbook}. 
 
 We restrict our attention to the  class of {\em Linear Structural Equation Models (LSEMs)}. A mixed graph $G$ gives rise to a linear structural equation model, which is the set of all joint distributions of a random vector $X=(X_1,\ldots, X_p)$ such that the variable $X_i$ associated to vertex $i$ is a linear function of a noise term $\epsilon_i$ and the variables $X_j$, as $j$ varies over the set of parents of $i$, denoted pa$(i)$, (i.e., the set of all vertices $j$ such that $j\to i \in E$). Thus,
$$X_i = \sum_{j\in\text{pa}(i)} b_{ij}X_j + \varepsilon_i, \,\, i\in V.$$

When the variables $X_1,\ldots, X_p$ are Gaussian, we are only able to recover the mixed graph up to Markov equivalence from observational data~\cite{handbook,  Lauritzen}.
When the variables are non-Gaussian, however, it is possible to recover the full graph from observational data. This line of work originated with the paper~\cite{Shimizu2006} by Shimizu et al. in which a linear non-Gaussian structural equation model corresponding to a direced acyclic graph (DAG), i.e., a graph without confounders, can be identified from
observational data using independent component analysis (ICA). Instead of ICA, the subsequent DirectLiNGAM~\cite{Shimizu2011} and Pairwise LiNGAM~\cite{PairwiseLiNGAM} methods use an
iterative procedure to estimate a causal ordering; Wang and Drton~\cite{WangDrton}
give a modified method that is also consistent in high-dimensional settings
in which the number of variables $p$ exceeds the sample size $n$.

Hoyer et al.~\cite{Hoyer} consider the setting where the data is generated by
a linear non-Gaussian acyclic model (LiNGAM), but some variables are unobserved. Their method, like ours, recovers  an acyclic mixed graph with  multidirected edges. However, it uses overcomplete ICA and requires the number of latent
variables in the model to be known in advance. Furthermore, current  implementations of
overcomplete ICA algorithms often suffer from local optima and can't guarantee convergence to a global one. To avoid using overcomplete ICA while still identifying unobserved confounders, Tashiro et al.~\cite{Parcellingam} propose a procedure, called ParcelLiNGAM, which  tests subsets  of observed variables. Wang and Drtong~\cite{BANG},  however,  show that this procedure works whenever the graph is ancestral, and propose a  new procedure,  called BANG, which uses patterns in higher-order moment data, and can identify {\em bow-free acyclic mixed  graphs}, a set of  mixed graphs much larger than the set of ancestral graphs. The BANG algorithm, however, recovers a mixed graph which only contains directed and bidirected edges.

Our method builds on the BANG procedure. We can identify a bow-free acyclic mixed graph, and, in addition, join bidirected edges together into a multi-directed edge whenever more than two of the observed variables have a hidden common cause.

The rest of the paper is organized as follows. In Section~\ref{sec:2} we give the necessary background on mixed graphs, linear structural equation models, and multi-treks. In Section~\ref{sec:3}, we present our algorithm (MBANG) and we prove that it recovers the correct graph as long as the empirical moments are close enough the the population moments. In Section 4 we  present numerical results, including simulations of different graphs and error distributions, and an application of our algorithm to a real dataset. We conclude with a short discussion in Section 5.

\section{Background}\label{sec:2}
In this section we introduce the key concepts that we use throughout the paper, including mixed graphs with multi-directed edges, linear structural equation models, multi-treks and cumulants.

%[introduce mixed graphs with multi-directed edges, bow-free]
%[cite BANG paper results]
%[cite theorem about cumulants...]

\subsection{Mixed graphs with multi-directed edges}

The notion of a mixed graph is widely used in graphical modelling, where bidirected edges depict unobserved confounding. In this paper a mixed graph is also allowed to contain multi-directed edges to depict slightly more complicated unobserved confounding.

\begin{definition}
(a). We call $G = (V, \mathcal D, \mathcal H)$ a {\em mixed graph}, where $V = \{1,\ldots, p\}$ is the set of vertices, $\mathcal D\subseteq V\times V$ is the set of directed edges, and $\mathcal H$ is the set of multi-directed edges (all $k$-directed edges for $k\geq 2$). (see part (b)).

(b). For $k\geq 2$, a {\em $k$-directed} (or {\em multi-directed}) edge between distinct nodes $i_1,\ldots, i_k\in V$, denoted by $(i_1,\ldots, i_k)$, is the union of $k$ directed edges with the same hidden source and sinks $i_1,\ldots, i_k$ respectively. Multi-directed edges are unordered.

\begin{figure}[H]
\begin{subfigure}{0.5\textwidth}
\begin{center}
\vspace{0.35cm}

    \includegraphics[width=0.6\textwidth]{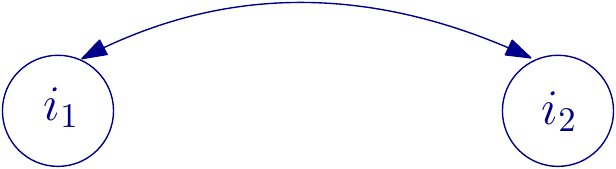}
    \caption{A 2-directed (or bidirected) edge between $i_1,i_2$.}
    \label{fig:2-edge}
\end{center}
\end{subfigure}
\begin{subfigure}{0.5\textwidth}
\begin{center}
    \includegraphics[width=0.6\textwidth]{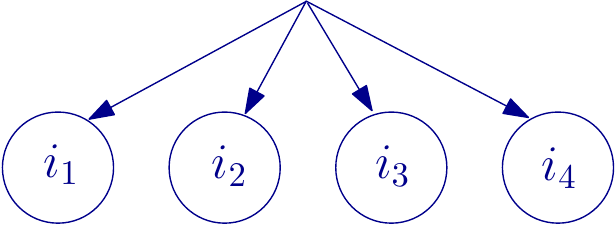}
    \caption{A 4-directed edge between $i_1,i_2,i_3,i_4$.}
    \label{fig:4-edge}
\end{center}
\end{subfigure}
\caption{Examples of multi-directed edges.}
\end{figure}
\end{definition}

Our method is able to recover {\em bow-free} {\em acyclic} mixed graphs. A mixed graph is bow-free if it does not contain a {\em bow}, and a bow consists of two vertices $i,j\in V$ such that there is both a directed and a multidirected edge between $i$ and $j$. In other words, $i\to j\in\mathcal D$ and there exists $h\in\mathcal H$ such that $i,j\in h$. A mixed graph is acyclic if it does not contain any {\em directed cycles}, where a directed cycle is a sequence of directed edges of the form $i_1\to i_2, i_2\to i_3,\ldots, i_\ell \to i_1$.
\begin{figure}[H]
\begin{subfigure}{0.5\textwidth}
\begin{center}
    \includegraphics[width=0.4\textwidth]{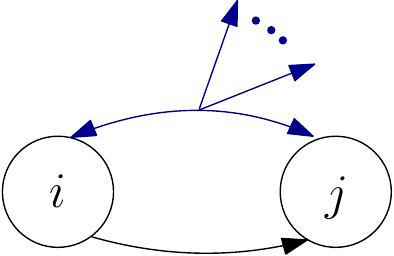}
    \caption{A bow between $i$ and $j$.}
    \label{fig:bow}
\end{center}
\end{subfigure}
\begin{subfigure}{0.5\textwidth}
\begin{center}
    \includegraphics[width=0.4\textwidth]{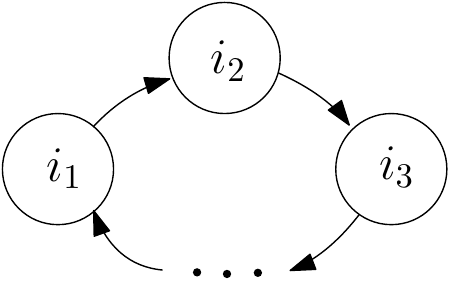}
    \caption{A cycle.}
    \label{fig:cycle}
\end{center}
\end{subfigure}
\caption{Examples of a bow and a cycle.}
\end{figure}

Acyclic mixed graphs with potential multidirected edges are all one  can hope to  recover from observational data.  Suppose for  a moment that  we  have data coming from a directed acyclic graph (DAG) where a subset  of the  variables is unobserved, e.g., consider the DAG in Figure~\ref{fig:originalModel},  where we only observe  variables  1, 4, and 5.

Hoyer et al.~\cite{Hoyer} show that any directed acyclic graphical model in which some of the variables are unobserved is observationally and causally equivalent to a unique {\em canonical model}, where a canonical model is a non-Gaussian Linear Structural Equation Model corresponding to an acyclic mixed graph $G=(V, \mathcal D, \mathcal H)$ such that none of the latent variables have any parents, and each latent variable has at least two children.  This means, that the  distribution of the observed variables in the original model is identical to that in the canonical model, and causal relationships of observed variables  in both models are identical. Therefore, we can focus our attention on the set of canonical models,  which can be represented precisely by acyclic mixed graphs with multi-directed edges. The graph in Figure~\ref{fig:canonical} is the  canonical model corresponding to the one  in Figure~\ref{fig:originalModel}.
\begin{figure}[H]
\begin{subfigure}{0.5\textwidth}
\begin{center}
    \includegraphics[scale = 0.6]{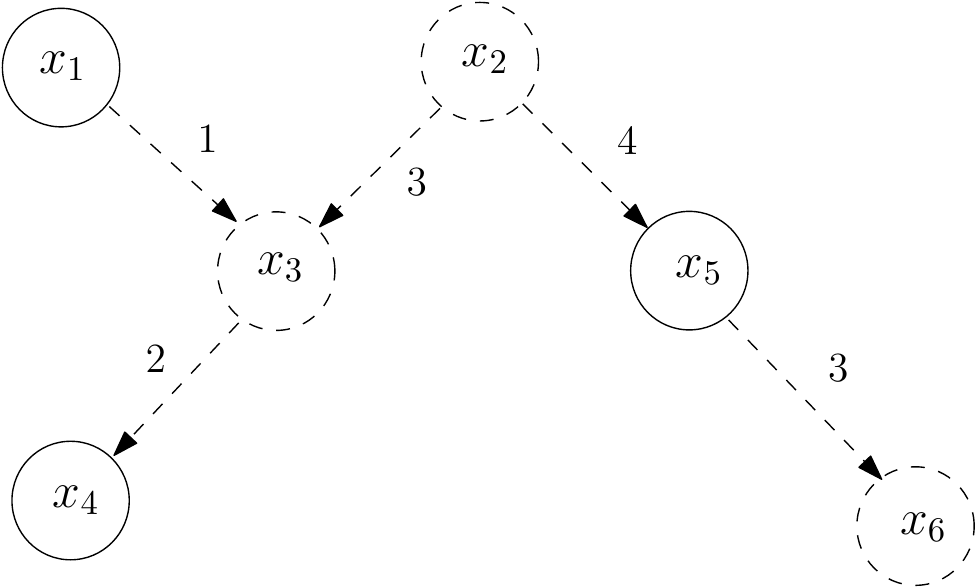}
    \caption{Original model}\label{fig:originalModel}
    \label{fig: origin}
\end{center}
\end{subfigure}
\begin{subfigure}{0.5\textwidth}
\begin{center}
    \includegraphics[scale = 0.6]{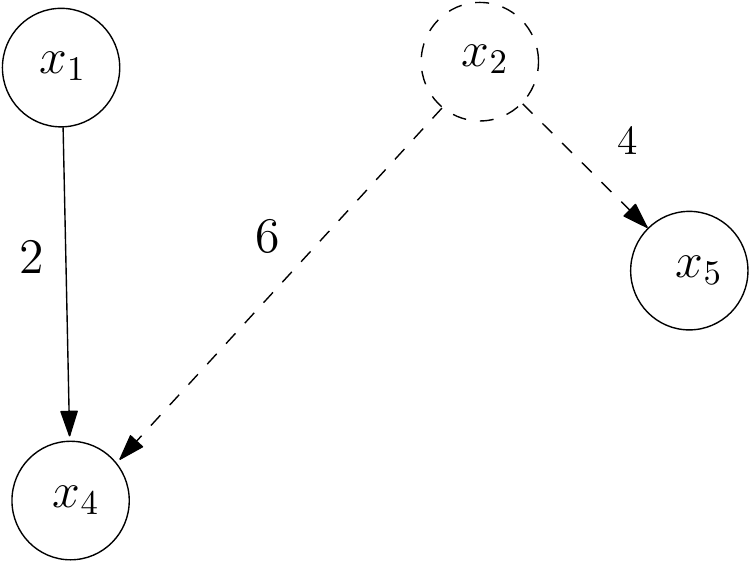}
    \caption{Canonical model}
    \label{fig:canonical}
\end{center}
\end{subfigure}
\caption{An example of two observationally and causally equivalent models from~\cite{Hoyer}.}
\label{fig:canonExample}
\end{figure}

\subsection{Linear Structural Equation Models}
Let $G=(V, \mathcal D, \mathcal H)$ be an acyclic mixed graph as defined in the previous section. It induces a statistical model, called a {\em linear structural equation model}, or {\em LSEM}, for the joint distribution of a collection of random variables $(X_i, i \in V )$, indexed by the graph’s vertices. The model hypothesizes that each variable
is a linear function of its parent variables and a noise term $\varepsilon_i$:
\begin{align}\label{LSEM}X_i = b_{0i} + \sum_{j\in\text{pa}(i)} b_{ji}X_j + \varepsilon_i
, i \in V,
\end{align}
where $\text{pa}(i) = \{j\in V: j\to i\in\mathcal D\}$ is the set of parents of vertex $i$. The variables $\varepsilon_i, i\in V$ are assumed to have mean 0, and the coefficients $b_{0i}$ and $b_{ji}$ are unknown real parameters. Since we can center the variables $X_i$, we assume that the coefficients $b_{0i}$ are all equal to 0. In addition, the multi-directed edge structure of $G$ defines dependencies between the noise terms $\varepsilon_i$, that is, if $i$ and $j$ are not connected by a multi-directed edge, then $\varepsilon_i$ and $\varepsilon_j$ are independent variables. In particular, if $G$ is a DAG (directed acyclic graph), i.e., it does not have any multi-directed edges, then all noise terms $\varepsilon_i$ are mutually independent.
Typically termed a system of structural equations, the system~\eqref{LSEM} specifies cause-effect relations whose straightforward
interpretability explains the wide-spread use of the models~\cite{Pearl, SGS00}.

We can rewrite the system~\eqref{LSEM} as
$$X = BX + \varepsilon,$$
where $X = (X_1,\ldots, X_p)^T$, $B = (b_{ij})$ is the coefficient matrix satisfying $b_{ij} = 0$ whenever $i\to j\not\in \mathcal D$, and $\varepsilon$ is the noise vector. Note that since $G$ is assumed to be acyclic, we can permute the coefficient matrix $B$ so that it is a lower triangular matrix. Therefore, the matrix $I-B$ is invertible, and the system~\eqref{LSEM} can further be rewritten as
$$X = (I-B)^{-1}\varepsilon.$$

\subsection{Cumulants and the multi-trek rule}
We recall the notion of a cumulant tensor for a random vector~\cite{ComonJutten}.
\begin{definition}
Let ${Z}=(Z_1,\cdots,Z_p)$ be a random vector  of length $p$. The $k$-th cumulant tensor of ${Z}$ is defined to be a $p\times \cdots\times p$ ($k$ times) table, $\mathcal C^{(k)}$, whose  entry  at  position  $(i_1,\cdots, i_k)$ is 
\begin{align*}
    \mathcal C^{(k)}_{i_1,\ldots, i_k}=\sum_{(A_1,\cdots,A_L)}(-1)^{L-1}(L-1)!\, \mathbb{E}\left[\prod_{j\in A_i}Z_j\right]\cdots\mathbb{E}\left[\prod_{j\in A_L}Z_j\right],
\end{align*}
where the sum is taken over all partitions $(A_1,\cdots, A_L)$ of the set $\{i_1,\cdots, i_k\}$.
\end{definition}

Note that if each of the variables $Z_i$ has zero mean, then we can restrict the summing over partitions for which each $A_i$ has size at least two. For example: 
\begin{align*}
    \mathcal C^{(4)}_{i_1,i_2,i_3,i_4}&=\mathbb{E}(Z_{i_1}Z_{i_2}Z_{i_3}Z_{i_4})-\mathbb{E}(Z_{i_1}Z_{i_2})\mathbb{E}(Z_{i_3}Z_{i_4})-\mathbb{E}(Z_{i_1}Z_{i_3})\mathbb{E}(Z_{i_2}Z_{i_4})+\\
    &\quad \mathbb{E}(Z_{i_1}Z_{i_4})\mathbb{E}(Z_{i_2}Z_{i_3})
\end{align*}

We now recall the notion of a multi-trek from~\cite{multiTrek}.
\begin{definition}
A $k$-trek in a mixed graph $G=(V, \mathcal D, \mathcal H)$ between $k$ nodes $i_1,i_2,\cdots, i_k$ is an ordered collection of $k$ directed paths $(P_1,\cdots, P_k)$ where $P_j$ has sink $i_j$ and either $P_1,\cdots, P_k$ have the same source of vertex, or there exists a multi-directed edge $h\in\mathcal H$ such that the sources of $P_1,\ldots, P_k$ all lie in $h$.
\end{definition}

\begin{figure}[H]
\begin{subfigure}{0.5\textwidth}
    \centering
    \includegraphics[width=0.5\textwidth]{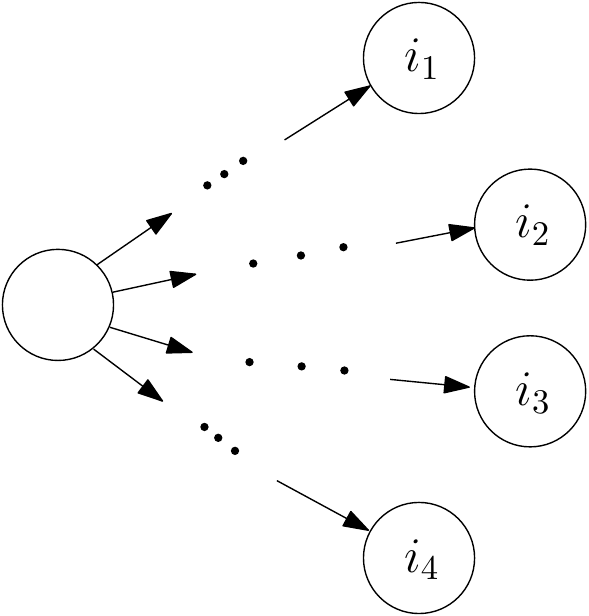}
    \caption{A 4-trek between $i_1,i_2,i_3,i_4$.}
    \label{fig:my_label}
    \end{subfigure}
    \begin{subfigure}{0.5\textwidth}
    \centering
    \includegraphics[width=0.45\textwidth]{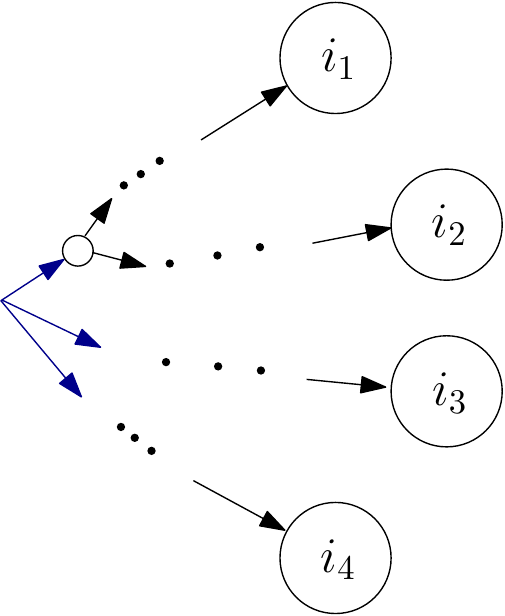}
    \caption{A 4-trek between $i_1,i_2,i_3,i_4$.}
    \label{fig:my_label}
    \end{subfigure}
    \caption{Examples of multi-treks.}
\end{figure}

% \begin{definition}\label{definition_determinant_tensor}
% \noindent
% Let $T$ be an order-$k$ $n \times ... \times n$ tensor. Then, its determinant is
% \begin{equation}
% \text{det}(T) = \sum_{\sigma_2,..., \sigma_{k}\in\mathfrak{S}(n)} \text{sign}(\sigma_2) \cdots \text{sign}(\sigma_{k}) \prod_{i=1}^{n} T_{i,\sigma_2(i),...,\sigma_{k}(i)},
% \end{equation}
% where $\mathfrak{S}(n)$ is the set of permutations of the set $\{1,\ldots, n\}$.
% \end{definition}

% \begin{definition}
% Given a collection of $k$ sets of nodes $S_1,...,S_k\subseteq V$ such that $\#S_1 =...=\#S_k = n$, a {\em $k$-trek system} ${T}$ is a collection of $n$ $k$-treks between $S_1,...,S_k$ such that the ends of $T$ on the $i$-th side equal $S_i$. We define the {\em top} of this $k$-trek system, $top({T})$, to be the union of the tops of the $k$-treks. We allow repeated elements in $top({T})$. 
% A $k$-trek system $ T$ has {\em sided intersection} if there exist two $k$-treks $(P_1,\ldots, P_k)$ and $(Q_1,\ldots, Q_k)$ in $ T$ and a number $1\leq i\leq k$ so that the directed paths $P_i$ and $Q_i$ have a common vertex. We denote by $\widetilde{\mathcal T}(S_1,\ldots, S_k)$ the set of $k$-trek systems between $S_1,\ldots, S_k$ that have {\em no} sided intersection.
% \end{definition}

The following consequence of the {\em multi-trek rule}~\cite{multiTrek} connects cumulants of LSEMs to multi-treks.
\begin{theorem}[\cite{multiTrek}]\label{thm:multiTrek} Let $G = (V, \mathcal D, \mathcal H)$ be an acyclic mixed graph, and let $i_1,\ldots, i_k\in V$. Then, 
$$\mathcal C^{(k)}_{i_1,\ldots, i_k} = 0$$
for the $k$-th cumulant tensor $\mathcal C^{(k)}$ of any random  vector  whose  distribution  lies in the LSEM corresponding to $G$ if and only if there is no $k$-trek between $i_1,\ldots, i_k$ in $G$.
\end{theorem}
%Combining the results above, for any vertices $i_1,...,i_k$ of $V$,  $\mathcal{C}^{(k)}_{i_1,..., i_k} = 0$ if and only if $i_1,...,i_k$ does not have a common hidden variable.

\section{Main Result}\label{sec:3}
In this section we present our algorithm and we show that it recovers the correct graph given enough samples. In the chart below we illustrate how the algorithm works when applied to observational data coming from the graph in Figure~\ref{fig:1a}.

\begin{figure}[]
\centering
    \centering
    \includegraphics[scale = 0.6]{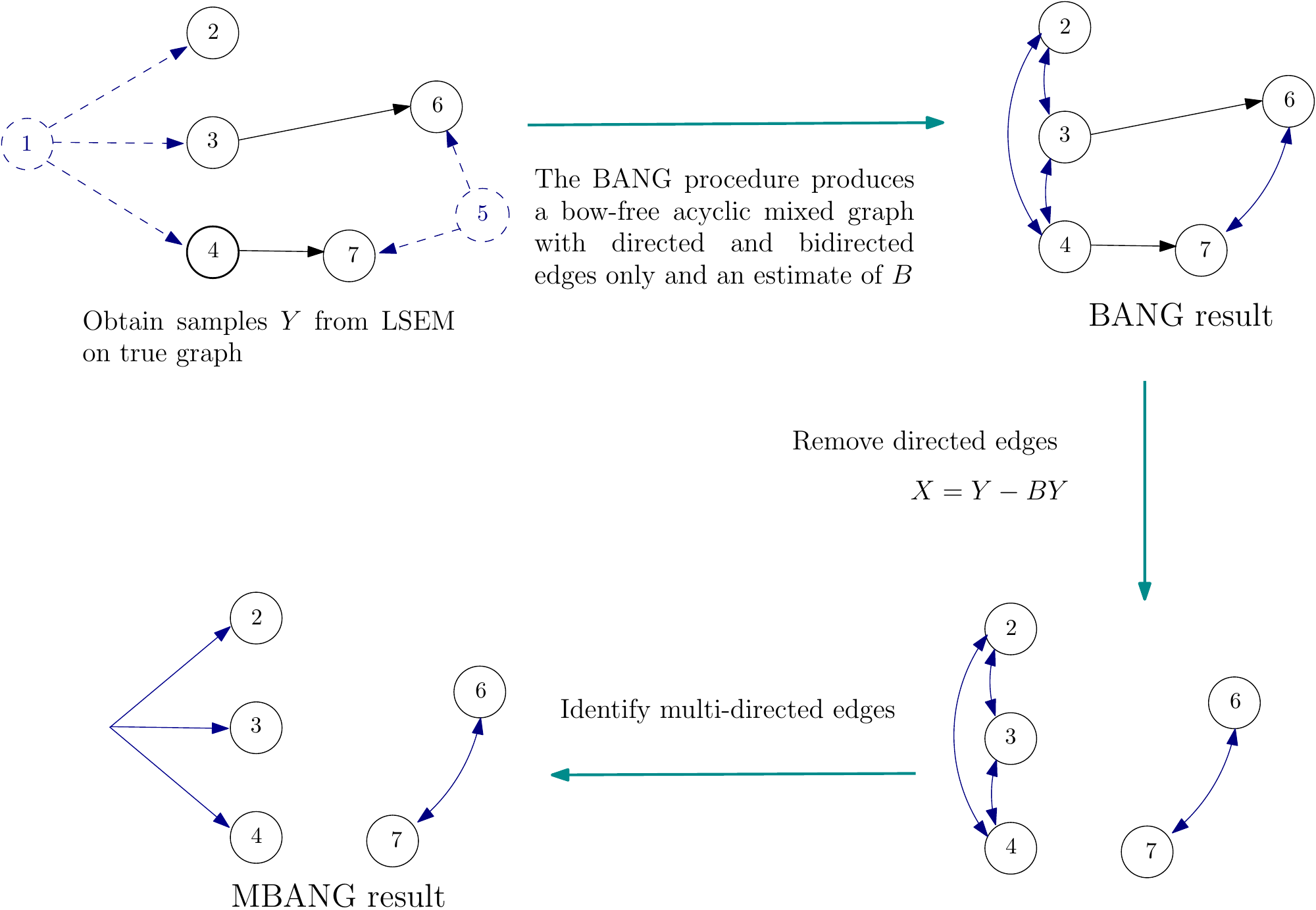}
    \caption{}
    \label{fig: flow}
\end{figure}

% Our algorithm can be summarized in two steps. The first step is to apply the \textit{Bang Procedure} to obtain a bow-free graph and then remove all the directed edges. The second step is to determine which bidirected edges in the \textit{Bang Procedure} are multidirected based on the output graph of the first step. \par

Given  a $p\times n$ data matrix $Y$ whose columns are i.i.d. sampels from a LSEM on an unknown bow-free acyclic mixed graph $G=(V, \mathcal D, \mathcal H)$ with an unknown direct  effects matrix $B$, we aim to recover the graph $G$  and  the matrix $B$.
The first step is to apply the BANG procedure~\cite{BANG} and obtain the coefficient matrix $B$ together with a bow-free acyclic mixed graph $G_1 = (V, \mathcal D, \mathcal B)$ which contains directed and bidirected edges only. The bidirected edges $\mathcal B$ are obtained from $\mathcal H$  by replacing each multi-directed edge $h=(i_1,\ldots, i_k)\in\mathcal H$ by $\binom k2$ bidirected eges, one for each pair $i_s, i_t\in\{i_1,\ldots,  i_k\}$. We call such a set $\mathcal B$ the {\em bidirected subdivision} of $\mathcal H$. We then, "remove" the directed edges by removing the direct effects given by $B$. This is done by replacing the original data matrix $Y$ with $X = Y-BY$. This new matrix $X$ can be thought of as observations from a LSEM corresponding to the acyclic mixed graph  $G'=(V, \emptyset, \mathcal H)$ which is the same as $G$ with the directed edges removed. However, we only know the bidirected subdivision $\mathcal B$ of $\mathcal H$. Finally, using the higher order cumulants of $X$ and a clique-finding algorithm on the  graph $G_1' = (V, \emptyset, \mathcal B)$, we identify all multi-directed edges in $\mathcal H$. The algorithm is summarized below. 

%In Figure \ref{fig: flow}, given the data from the true graph which has two hidden variables 1 and 5, we apply \textit{Bang Procedure} to obtain a bow-free mixed graph only containing bi-directed edges as the Bang result. To focus on multi-directed edges, the directed part is removed by taking $X=Y-BY$. Finally, by calculating higher cumulants, multi-directed edges can be correctly identified.  

\begin{algorithm}[H]
\caption{MBANG procedure}\label{Alg:1}
\begin{algorithmic}[1]
\STATE Input: $Y\in \mathbb{R}^{p\times n}$ arising from an unknown LSEM  on $G=(V, \mathcal  D, \mathcal H)$ with unknown direct effects matrix $B$.
\STATE Apply the BANG procedure to estimate the coefficient matrix $B$ and the  mixed graph $G_1=(V, \mathcal D,  \mathcal B)$.
\STATE Let $X=Y-BY$ be the new observation matrix, corresponding to the graph $G'=(V, \emptyset, \mathcal H)$ which has no directed edges.
\STATE Initiate $R=Q=\varnothing$, $P=\{1,2,...,p\}$.
\STATE Apply {Algorithm~\ref{alg:2}} to $X, G_1' = (V, \emptyset, \mathcal B), R,P$ and $Q$ to recover the set $\mathcal  H$  of multidirected edges.
\STATE Output: the  graph $G = (V, \mathcal D, \mathcal H)$  and the matrix $B$.
\end{algorithmic}
\end{algorithm}

Algorithm~\ref{alg:2} finds those multi-directed edges that contain the most vertices  and are consistent with  the cumulant structure of the data matrix $X$. It is based  on the Bron-Kerbosh algorithm~\cite{BronKerbosch} for finding all cliques in an undirected graph, applied to the bidirected edges in $G_1'  = (V, \emptyset, \mathcal B)$.

\subsection{The BANG procedure}
In this section we briefly describe the {BANG Procedure}~\cite{BANG}. It is an algorithm which takes as input  a $p\times n$ data matrix $Y$, and returns a \textit{bow-free acyclic mixed graph} $G$ which contains directed and bidirected edges only,  consistent with  $Y$ as well as a direct effects matrix $B$. \par
  Suppose the observed data is drawn from a LSEM whose corresponding graph is the one in Figure~\ref{fig:bangExample}. Figure \ref{fig:bangFlowChart} shows how the {BANG algorithm} works.
%Theoretically, by testing higher moments of the sample, two kinds of relations can be correctly identified. 

% \begin{itemize}
%     \item siblings (vertices joined by a bi-directed edge) and ancestors.
%     \item parents and non-parent ancestors
% \end{itemize}
%  \begin{figure}[H]
%      \centering
%      \includegraphics[scale = 0.5]{bangExample.pdf}
%      \caption{bang example}
%      \label{fig:bangExample}
%  \end{figure}

It first identifies sibling and ancestor relations, and then it distinguishes parent and non-parent ancestors. Recall that two vertices are {\em siblings} if  there is a multi-directed edge between them, a vertex $i$ is a {\em parent} of a vertex $j$ if there is a directed  edge $i\to j$,  and a vertex $i$ is an {\em ancestor} of a vertex $j$ if there is a direted  path $i\to i_0\to\cdots\to i_k\to j$.
 \begin{figure}[H]
     \centering
      \begin{subfigure}{0.25\textwidth}
      \vspace{5cm}
      \hspace{-0.2cm}
     \hspace{0.5cm}\includegraphics[width=0.6\textwidth]{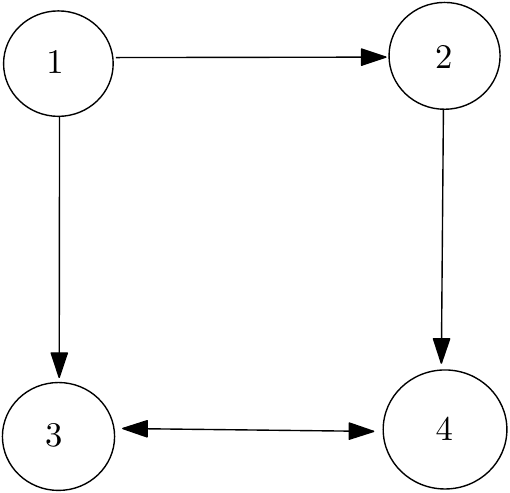}
     \caption{Graph example}
     \label{fig:bangExample}
     \end{subfigure}
     %\hspace{1.2cm}
     \begin{subfigure}{0.65\textwidth}
     \hspace{1cm}\includegraphics[width=0.95\textwidth]{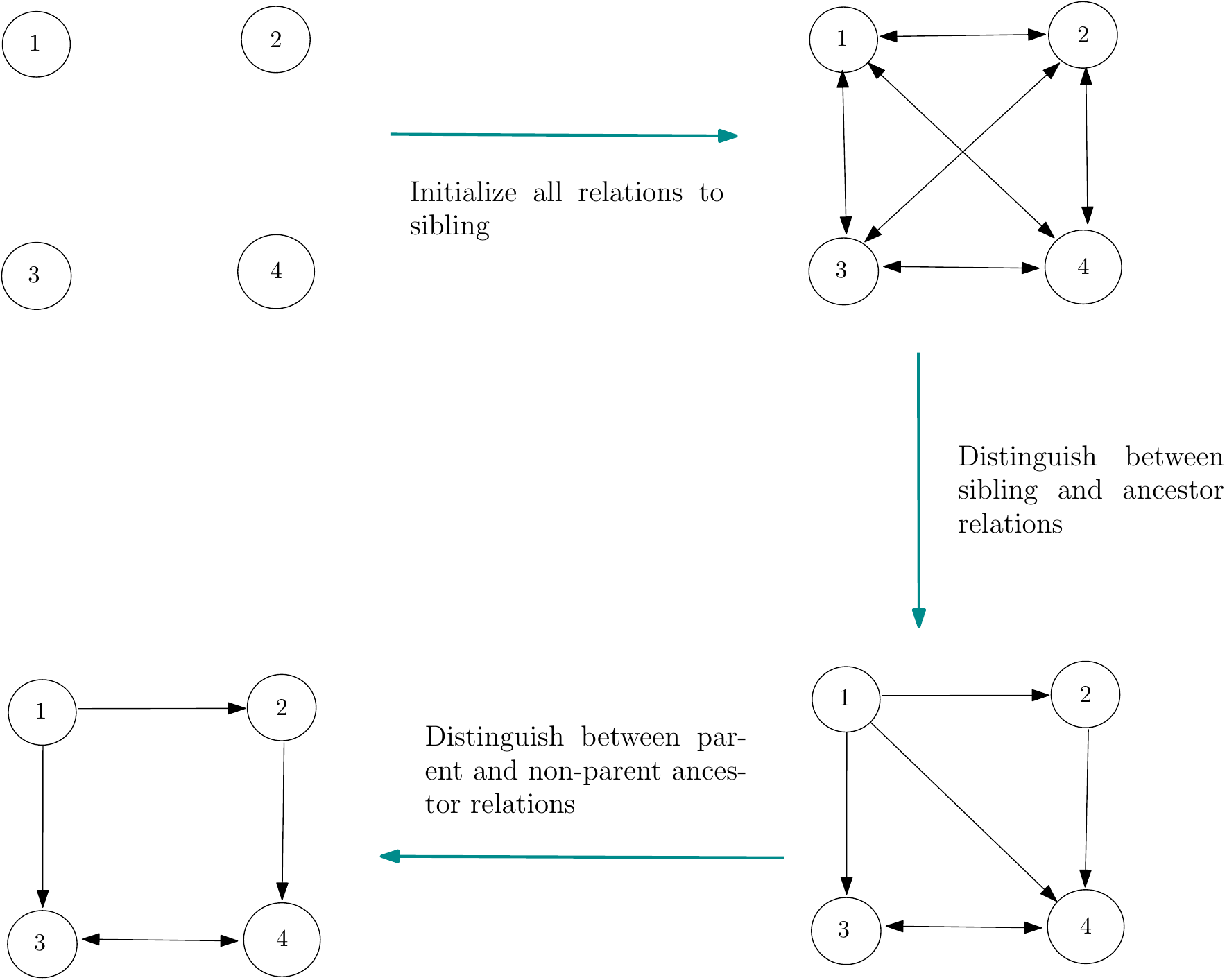}
     \caption{BANG flowchart}
     \label{fig:bangFlowChart}
     \end{subfigure}
    \caption{BANG algorithm example}
 \end{figure}
 
 However, the {BANG algorithm} returns {bow-free acyclic mixed graphs} in which the latent  variable structure is recorded using  bidirected edges only, i.e.,  it  recovers the bidirected subdivision of the true set of multi-directed edges $\mathcal H$. This means it cannot determine whether more than two vertices have a common cause. We resolve this problem using cumulant information.
 
 %For example,  in figure \ref{fig: bangLimit}, even though 2,3 and 4 do not have a common cause in case 1 and they do in case 2, the \textit{bang procedure} produces the same graph in both cases due to the limitation of the output. To distinguish between these graphs, we calculate the higher order cumulants of 2,3 and 4. 
 
 \begin{figure}[]
     \begin{subfigure}{\linewidth}
     \centering
     \includegraphics[scale = 0.5]{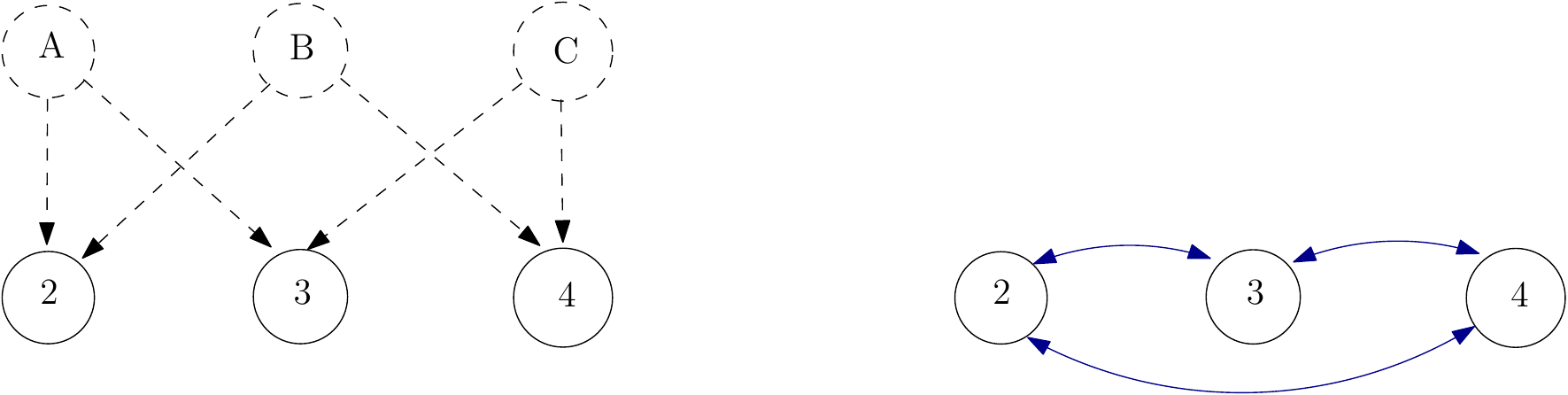}
     \caption{Case 1}
     \end{subfigure}
     \begin{subfigure}{\linewidth}
     \centering
     \includegraphics[scale = 0.5]{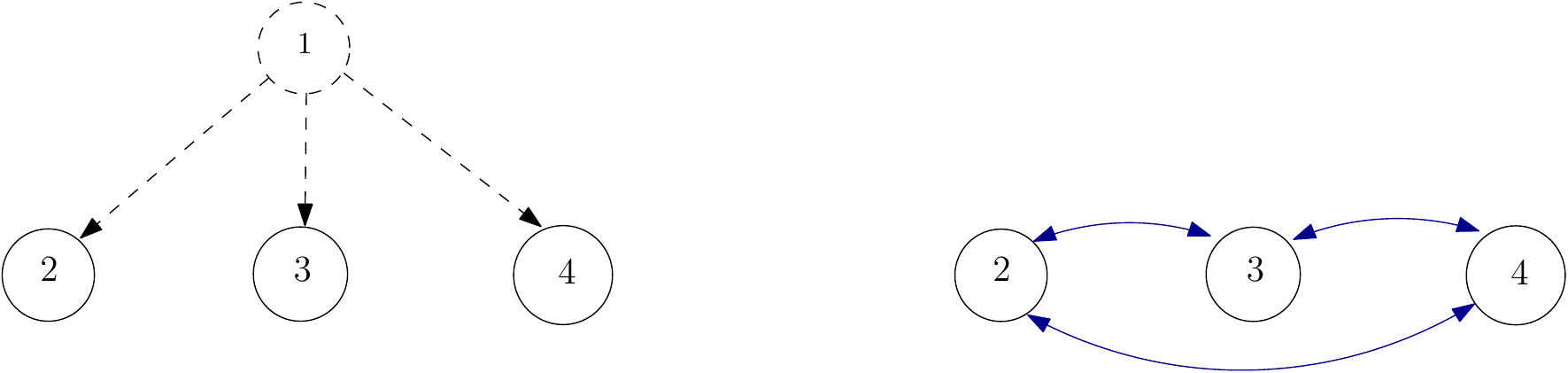}
     \caption{Case 2}
     \end{subfigure}
     \caption{The  same bidirected  edge structure depicts two different hidden  variable  models.}
     \label{fig: bangLimit}
 \end{figure}

\subsection{Finding the  multidirected edge  structure}
Consider the two models in Figure \ref{fig: bangLimit}. By Theorem \ref{thm:multiTrek}, we  know that $\mathcal{C}^{(3)}_{2,3,4}=0$ in Case 1 whereas $\mathcal{C}^{(3)}_{2,3,4}\neq 0$ in Case 2. In other words, the exact model can be recovered by testing whether $\mathcal{C}^{(3)}_{2,3,4}=0$. \par

% More generally, suppose we are given the output of \textit{bang procedure} $G$ (bow-free mixed graph only containing bi-directed and directed edges). If two vertices have a common hidden parent, there must be a bi-directed arrow between them. Therefore, there is no need to consider the directed part of $G$. By calculating the residue $X=Y-BY$, we restrict our attention on the bi-directed part $G^*$. 

More generally, as depicted in Figure~\ref{fig: flow}, we first apply the BANG algorithm to the data marix $Y\in\mathbb  R^{p\times n}$ arising from a LSEM on $G=(V, \mathcal D, \mathcal H)$, to obtain an  estimate of the direct effects matrix $B$ and the bow-free acyclic mixed graph $G_1=(V, \mathcal D,  \mathcal  B)$ with  directed and bidirected edges only, where there is a bidirected edge  between $i$  and $j$ in $G_1$ if and only if there is a multidirected edge in $G$ containing both $i$ and $j$, i.e., $\mathcal B$ is  the bidirected subdivision of $\mathcal H$. Then,  we form a new data matrix $X=Y-BY$ and we consider the graph $G' = (V, \emptyset, \mathcal H)$ which contains only multi-directed edges. In other words, the  data matrix $X$ can be thought of as a matrix of samples coming from a LSEM on the graph $G'$. However, we  only know the  graph $G_1'  =  (V, \emptyset, \mathcal B)$.
The bidirected edges $\mathcal B$ recovered by the  BANG algorithm can now be "merged" together to obtain  the true  multi-directed edge structure $\mathcal H$.
%given a \textit{bow-free} acyclic mixed graph $G'$ containing bi-directed edges, we aim to find its multi-directed edge structure. If two vertices have a common hidden parent, there must be a bi-directed edge between them. We can focus on the bi-directed part of $G'$ by calculating the residue $X=Y-BY$.\par
%Our goal now is to merge some of the bidireced edges in  $G'$ into multidirected edges so as to account for  the correct hidden variable information. 
Since $G_1'$ contains only bidirected edges,
we look for all cliques (of  bidirected edges) $\{i_1,\ldots, i_k\}$ in $G'$ such that  $\mathcal C^{(k)}_{i_1,\ldots, i_k}\neq 0$,  where $\mathcal C^{(k)}$ is the $k$-th cumulant of $X$. We do this by  adapting the  Bron-Kerbosch algorithm~\cite{BronKerbosch}, which is used for finding all cliques  in an undirected graph. 

\begin{algorithm}[H]
	\caption{Determine Multidirected Edges}
	 \label{alg:2}
\begin{algorithmic}[1]
\STATE Input: $X\in \mathbb{R}^{p\times n}, G_1'=(V, \emptyset, \mathcal B), R,P$ and $Q$. Denote the elements of $R$ to be $i_1,...,i_k$.
\IF{$P$ and $Q$ are both empty}
\STATE Report $R$ as a multidirected edge.
\ENDIF
\FORALL{vertex $v\in P$}
\IF{$N(v) \neq \varnothing$, where $N(v)$ is the set of vertices adjacent to $v$ in $\mathcal B$} 
\IF{$\mathcal{C}^{(k+1)}_{i_1,\cdots,i_k,v} \neq 0$ or $\exists$ $j \in \{i_1,\cdots,i_k\}$ s.t. $\mathcal{C}^{(k+2)}_{i_1,\cdots,i_k,v,j}\neq 0$}\label{lst: line: condition}
\STATE call Algorithm~\ref{alg:2}$(X, R\cup \{v\}, P\cap N(v),Q\cap N(v))$
\ENDIF
\ENDIF
\STATE $P=P\backslash \{v\}$
\STATE $Q = Q\cup \{v\}$
\ENDFOR
 \end{algorithmic}
 \label{alg:multiedge}
 \end{algorithm}
 
Algorithm~\ref{alg:2} is a direct modification of the Bron-Kerbosch algorithm~\cite{BronKerbosch}. It is  a  recursive algorithm, that maintains three disjoint sets of vertices $R, P, Q \subset V$, and aims to output all cliques in the  graph which contain  all vertices in $R$, do not contain any  of the  vertices in $Q$, and  could contain some  of the vertices in $P$. The only addition to the Bron-Kerbosch algorithm that Algorithm~\ref{alg:2} does is the test of whether or not specific cumulants are nonzero in Line 7.

% If $P$ and $Q$ are both empty, then we know that $R$ is the multidirected edge including all the possible vertices because there is no other vertex to consider. \par

% If $P$ is not empty, then we have to consider each vertex in $P$ to figure out if any of them can form a multidirected edge with $R$. For the new $P$ and $Q$ passed to the recursive function call, we can only consider vertices in $N(v)$ to narrow down the vertices we need to consider. This is because the only possible vertices that can form multidirected edge with $v$ is $N(v)$, so the other vertices in the old $P$ and $Q$ are not relevant to the new function call. Hence, we pass $P\cap N(v)$ as the new $P$ and $Q\cap N(v)$ as the new $Q$. \par

In practice, since Theorem \ref{thm:multiTrek} holds for generic distributions, sometimes random variables consistent with a model that has a multi-directed edge  between $i_1,\ldots,  i_k$ may still have a zero cumulant. For example, if $i_1$, $i_2$, and $i_3$ are all caused by a hidden parent variable  whose distribution is symmetric around $0$, then $\mathcal{C}^{(3)}_{i_1,i_2,i_3}=0$. To solve this problem, we relax the criterion $\mathcal{C}^{(k)}_{i_1,\cdots,i_k} \neq 0$. In our implementation, we also check if there  exists $j\in\{i_1,\ldots, i_k\}$ such that $\mathcal{C}^{(k+1)}_{i_1,\cdots,i_k,j}\neq 0$. If  there is a multidirected edge $h$ such that  $i_1,\ldots, i_k\in h$ but due to non-genericity of the model $\mathcal C^{k}_{i_1,\ldots,  i_k} =0$, we find that often times $\mathcal{C}^{(k+1)}_{i_1,\cdots,i_k,j}\neq 0$ which allows us to reach the correct conclusion.

%treat $i_1,\cdots,i_k$ having a multidirected edge if $\exists$ $j \in \{i_1,\cdots,i_k\}$ s.t. $\mathcal{C}^{(k+2)}_{i_1,\cdots,i_k,v,j}\neq 0$. This new criterion will not identify the multidirected edges do not exists. The reason is that the new variables added to the cumulant is one of the original variables, which means the new set of variables, with one of them repeated, still does not have a multidirected edge, provided that the original set of variables does not have a multidirected edge. \par

% After this recursive call, we have already tried to add $v$ to $R$, so we move it from the subset of possible larger cliques, $P$, to $Q$, the subset of vertices that we no longer consider. \par

\subsection{Theoretical results}
In this section we show that Algorithm~\ref{Alg:1} recovers the correct bow-free acyclic mixed graph given enough samples. We begin with a  population moment result.

\begin{theorem}\label{thm: alg}
Suppose $Y$ is generated from a linear structural equation model corresponding to a bow-free acyclic mixed graph $G=(V, \mathcal D, \mathcal H)$. Then for a generic choice of the coefficient matrix $B$ and generic error moments, when given the population moments of $Y$, Algorithm~\ref{Alg:1} produces the correct graph $G$. 
\end{theorem}

\begin{proof}
Algorithm~\ref{Alg:1} first apples the BANG procedure to produce a \textit{bow-free acyclic mixed graph} $G_1 = (V, \mathcal D, \mathcal  B)$ containing bi-directed edges and directed edges only as well as the coefficient matrix $B$. In~\cite{BANG} the authors show that the BANG algorithm recovers the correct bow-free acyclic mixed graph (with directed and bidirected edges only) and the correct coefficient matrix $B$. Note that  this means that the graph $G_1$ will have a bidirected edge $b\in \mathcal B$ between two nodes $i$ and $j$ which are part of a multi-directed edge $h\in\mathcal H$ in $G$. Consider the data matrix $X=Y-BY$ which corresponds to a graph $G' = (V, \emptyset, \mathcal H)$ obtained by removing all directed edges in $G$, and the  graph $G_1'=(V, \emptyset, \mathcal B)$ obtained from $G_1$ by  removing all directed edges. Our task is to merge together some of the bidirected edges in $G_1'$ in order to obtain the graph $G'$.

There is  a multi-directed edge between $i_1,\ldots, i_k$ in $G'$ (or, equivalently, in $G$), if and only if any  two of $i_1,\ldots, i_k$ are joined by a bidirected edge in $G_1'$,  i.e., they will  form a clique, AND % By the property of cumulant, if the $k$ is maximum, then $v_1,\cdots, v_k$ form a maximum clique in $G^*$ satisfying.
\begin{align}\label{eq:nonZeroCumulant}
    \mathcal{C}^{(k)}_{i_1,\cdots, i_k}\neq 0,
\end{align}
where $\mathcal{C}^{(k)}$ is the $k$-th cumulant of $X$ (assuming the distribution of $X$ is generic).
Therefore, finding all multi-directed edges is equivalent to finding all maximal cliques in $G_1'$ for which~\eqref{eq:nonZeroCumulant} holds. This is precisely what Algorithm~\ref{alg:2} does -- it applies the Bron-Kerbosch algorithm for  finding all cliques in  an undirected graph (which can equally well be applied to the graph  $G_1'$ since it only  has  bidirected edges) and for each such  clique it checks whether~\eqref{eq:nonZeroCumulant} holds.

Therefore, Algorithm~\ref{Alg:1} produces the correct graph $G$.
\end{proof}

\begin{theorem}
Suppose $Y_1,Y_2,\cdots,Y_n$ are generated by a linear structural equation model which corresponds to a bow-free acyclic mixed graph $G=(V, \mathcal  D, \mathcal H)$. Then, for generic choices of the effects coefficient matrix $B$, and generic error moments, there exist $\delta_1,\delta_2, \delta_3>0$ such that if the sample moments are within a $\delta_1$ ball of the population moments of $Y$, then Algorithm~\ref{Alg:1} will produce the correct graph $G$ when comparing the absolute value of the sample statistics to $\delta_2$ as a proxy for the independence tests in the  BANG  procedure and when comparing the  absolute value of the cumulants to $\delta_3$ as a proxy in the tests for the vanishing of cumulants in Algorithm~\ref{alg:2}. 
\end{theorem}

\begin{proof}
First,  consider the data matrix $X=Y-BY$. Note that by definition the cumulants of $X$ have the form
\begin{align*}
    \mathcal C^{(k)}_{i_1,\ldots,i_k}  :=  cum\left(X_{i_1},\cdots,X_{i_k}\right)=\sum_{(A_1,\cdots,A_L)}(-1)^{L-1}(L-1)!\,\mathbb{E}\left[\prod_{j\in A_1}X_j\right]\cdots\mathbb{E}\left[\prod_{j\in A_L}X_j\right],
\end{align*}
which is a rational function of the moments of $X$,  which are rational functions of the moments of $Y$ and of the matrix $B$. Thus, it is also a continuous function of the population moments of $Y$. For the population cumulants $\mathcal{PC}^{(k)}$ of $X$  let
\begin{align*}
    \delta_3 = \frac12 \min_{\mathcal{PC}^{(k)}_{i_1,\cdots,i_k}\neq 0}\,\left|\mathcal{PC}^{(k)}_{i_1,\cdots,i_k}\right|,\quad 1\leq k\leq p.
\end{align*}
Thus, there exists $\delta'>0$ such that whenever the empirical moments of $Y$ are within a $\delta'$ ball of its population moments and the estimated $B$ is within $\delta'$ of the true one, all of the estimated cumulant entries of $X$ are within an $\delta_3$ ball of the entries of its population cumulants $\mathcal C^{(k)}$.

Next, we know from~\cite{BANG} that there exist $\delta,\delta_2 > 0$ such that if  the sample moments are within a $\delta$ ball of the population moments BANG will output the  correct graph $G_1=(V, \mathcal D, \mathcal B)$ (where $\mathcal B$ is the bidirected subdivision  of $
\mathcal H$) if it uses $\delta_2$  as a proxy for its independence tests.
%{\color{blue}
Furthermore, if the sample moments are within a $\delta$ ball of the  population moments, the estimated directed effects matrix $B$ will be  within a $\delta'$ ball from the  true  direct  effects matrix. This is because the  estimated matrix $B$ is a rational function of the sample moments~\cite{BANG}.

%Each step of updating $D$ contains two stages in BANG procedure.  Determing which entries of $D$ should be updated, and update it with 
% \begin{align*}
%     D_{C, v} := ((I-D)_{C,A}S_{A,C})^{-1}(I-D)_{C,A}S_{A,v}
% \end{align*}
%which is a rational function of the covariance matrix $S$. (Since $D$ in the formula is updated inductively by $S$). Provided the the entries to be updated are correctly identified, the calculation will produce the exact correct coefficient matrix $B$ if the covariance matrix $S$ is exactly correct. So we can see the entries of the estimated coefficent matrix $D$ is a rational function of the covariance matrix $S$. Meanwhile, to identify these entries correctly, BANG procedure is testing $\mathbb{E}(\gamma_v^{K-1}\gamma_u)$ at every updating point.
%This is also a rational function of $Y$ and $S$. Since we only need to test finitely many this moments, therefore, if the moment of $Y$ is closed enough to the population moments, we can guarantee all test results are correct. Then the entries of the estimated $D$ is a rational function of $S$. }

Thus, choosing $\delta_1 =  \min(\delta, \delta')$, we see  that if the sample moments of $Y$ are within $\delta_1$ of its population moments, and if we use $\delta_2$ as a proxy for its independence tests in BANG and $\delta_3$ as as a proxy in the tests for vanishing cumulants, Algorithm~\ref{Alg:1} will yield the correct graph $G=(V, \mathcal D, \mathcal H)$.
%By the Bang procedure, given $\epsilon>0$, there exists $\delta>0$ such that whenever the sample moments of $Y$ are within the $\delta$-ball of the population moments, it can produce the correct graph containing bi-directed edges and directed edges $G'$ and all entries of the estimated $D'$ is within $\epsilon$-ball of $D$. Therefore, $\delta$ can be chosen satisfying the moments of $Y-D'Y$ ($D'$ is the estimated coefficient matrix from the Bang procedure) are within $\delta'$-ball of the population moment. 
%Now when the sample $Y$ is within the $\delta$-ball of the population momment, all the test in both algorithms will produce correct result under this error $\delta$, which means that \textbf{algorithm 1} can produce the desired results.
\end{proof}

\section{Numerical Results}\label{sec:4}
In  this section we examine how  well Algorithm~\ref{Alg:1} performs numerically.
\subsection{Simulations}
We carried out a number of simulations using four different types of error distributions: the uniform distribution on $[-10,10]$; the student's $t$-distribution with 10 degrees of freedom; the Gamma distribution with shape $=2$ and rate $=4$; and the chi-squared distribution with 2 degrees of freedom.
We shift these distributions to have zero mean. The $t$-distribution is used to test the performance of the algorithm  in cases when the distribution resembles the Gaussian distribution.  \par

We generate random bow-free acyclic mixed graphs as follows. First, we uniformly select a prescribed number of directed edges from the set $\{(i,j)|i<j\}$, and we choose the  direct effect coefficients for each edge uniformly from $(-1,-0.6)\cup(0.6, 1)$. Afterwards, some of the vertices of this directed acyclic graph which have no parents are regarded as unobserved variables, and bow structures caused by marginalizing these vertices are removed from the graph. The resulting graph is a bow-free acyclic mixed graph with  potential multidirected edges.\par
To evaluate the performance of our method, we calculate the proportion of time that the algorithm recovers the true graph.
%, as well as the proportion of time that the algorithm correctly finds the correct multidirected edge structure given that the  true graph (with multi-directed  edges  replaed by cliques of bidirected edges) is recovered by the BANG algorithm. 
We generate graphs with  7 vertices, and after some of them are made hidden, the final graphs usually contain 5 or 6 observed variables. We test three settings with different levels of sparsity. Sparse graphs have 5 directed  edges before marginalization, medium graphs have 8, and dense graphs have 12. We do not restrict the number of multidirected edges, however, almost all dense graphs have multidirected edges, and most of the medium graphs have at least one multidirected edge.\par

In the first step of our algorithm, we perform the BANG procedure, and we set all the nominal levels of the hypothesis test to $\alpha = 0.01$, as suggested by the authors of the BANG algorithm\cite{BANG}. The tolerance value used in our cumulant tests we use is $0.05$.\par

For each of the three settings: sparse, medium, and dense, we tested 100 graphs by taking different numbers of samples: 10000, 25000 and 50000. 
In Figure~\ref{fig:test_edge_percent}, we show the percent of correctly identified  bidirected and multidirected edges for each setting. The $x$ axis represents the sample size and the $y$ axis represents the percentage.\par 

In Figure \ref{fig:test_edge_total}, we show the proportion of graphs that were recovered precisely by the BANG and MBANG algorithms. For the BANG algorithm, we recognize each multidirected edge as the corresponding set of bidirected edges.

In practice, we normalize the data matrix $X$ by dividing  each row by its standard deviation before initiating Algorithm~\ref{alg:2} in order to control the variability of its cumulants. 
\begin{figure}[H]
    \centering
    \includegraphics[width=\textwidth]{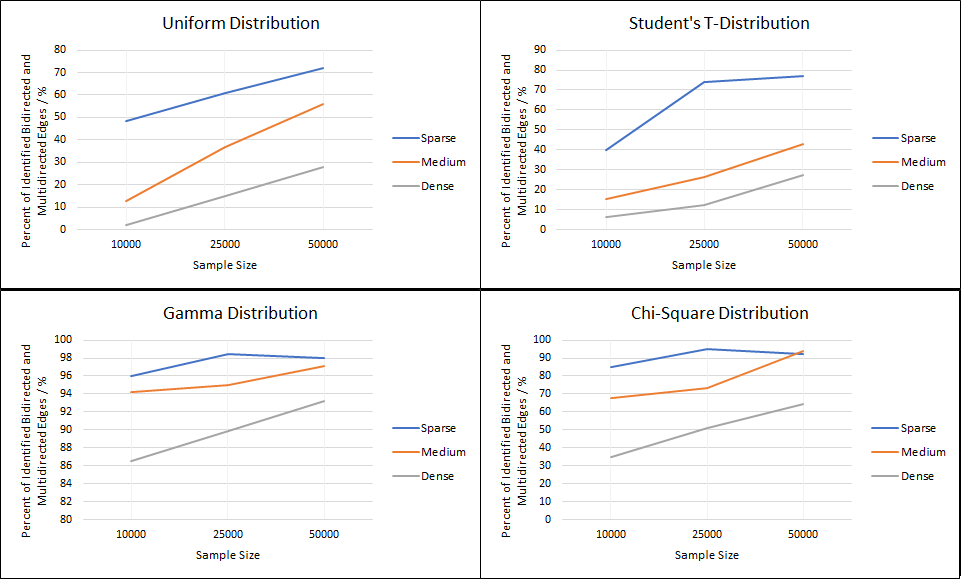}
    \caption{Random Graph Results - Correct Multidirected Edges} 
    \label{fig:test_edge_percent}
\end{figure}

\begin{figure}[]
    \centering
    \includegraphics[width=\textwidth]{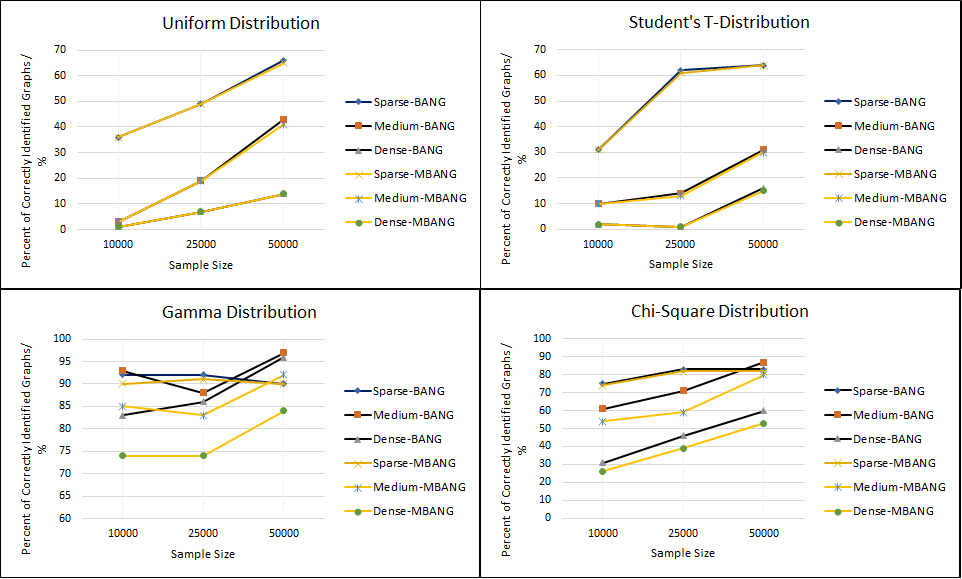}
    \caption{Random Graph Results - Correct Graphs} 
    \label{fig:test_edge_total}
\end{figure}

Among all $3600$ graphs we tested in this data set, the BANG algorithm recoverd $1883$ correctly, and the MBANG algorithm recovered $1774$ correctly. Hence, given that the BANG result was correct, the MBANG algorithm identified $94.2\%$ of the graphs correctly. Among all dense graphs, this proportion is reduced to $88.0\%$. However, this drop in accuracy might be expected because in our simulation dense graphs  contain more multi-directed edges.

%We also tested the performance of the MBANG algorithm under different tolerance values in the cumulant test. In this set of tests, since we are more interested in selecting the best tolerance, only dense graphs were tested because they have more multidirected edges in general.

\subsection{A Real Data Set}
In a paper by Grace et al. \cite{Grace}, a LSEM was used to examine the relationships between land productivity and the richness of plant diversity, the full model of which is shown in Figure~\ref{fig:Grace}. Wang and Drton~\cite{BANG} choose a subset of the variables and consider the model in Figure~\ref{fig:targetBAP} as the ground truth model. Figure \ref{fig:discovered} shows the graphical model they discover using the BANG procedure with nominal test level 0.01.

Since this is a real data set, it is possible that some of the hidden variables affect more than one bidirected edge, or there exist other hidden variables that can affect the hidden variables detected in the BANG procedure. After applying our MBANG algorithm with 0.05 tolerance, we found that all bidirected edges can be grouped in three 3-directed  edges: ("PlotSoilSuit", "PlotProd",  "SiteBiomass"),  ("PlotSoilSuit", "SiteBiomass",  "SiteProd"), and ("PlotSoilSuit", "SiteBiomass", "SiteProd"),  see Figure~\ref{fig:discoveredMBANG}.

%This is somehow expected becasuse in the real data, there might always be some confounding variables that cannot be controlled in the experiment.

\begin{figure}[H]
    \hspace{-0.6cm}
    \begin{subfigure}{0.4\textwidth}
    \centering
    \includegraphics[width=0.9\textwidth]{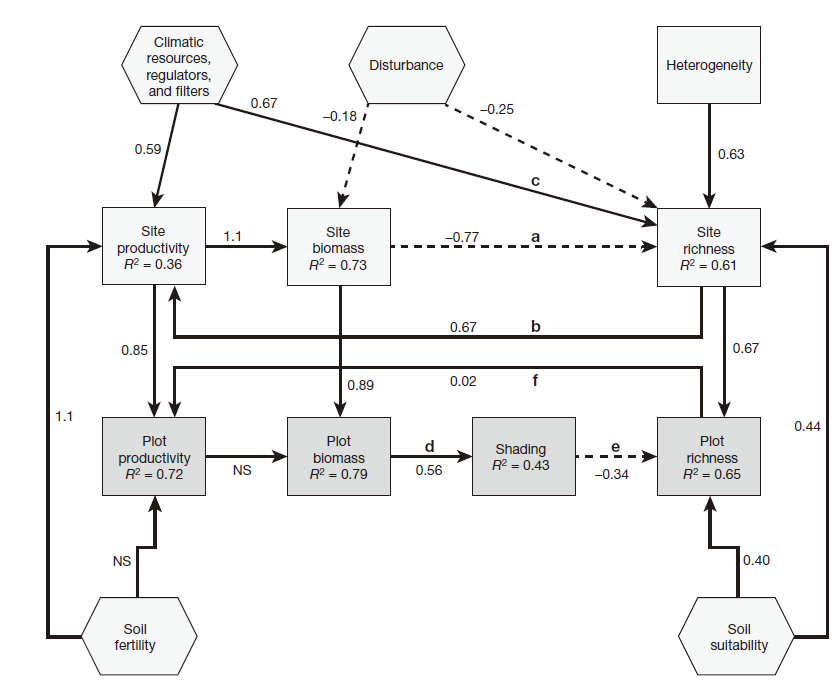}
    \caption{Full model from \cite{Grace}} 
    \label{fig:Grace}
    \end{subfigure}
    \hspace{0.1cm}
\begin{subfigure}{0.6\textwidth}
\centering
\vspace{1.6cm}
	\begin{tikzpicture}[->,shorten >=1pt,
auto,
main node/.style={rounded corners,inner sep=0pt,fill=gray!20,draw,font=\sffamily,
	minimum width = 1.5cm, minimum height = .5cm, scale=0.95}]

	\node[main node] (Pprod) [right = 0cm] {Pl Prod};
	\node[main node] (Pbio) [right = 2cm]  {Pl Bio};
	\node[main node] (Shade) [right =4cm]  {Pl Shade};
	\node[main node] (Prich) [right = 6cm]  {Pl Rich};
	\node[main node] (Sprod) [above = 1.1
	cm,right = 0cm]  {St Prod};
	\node[main node] (Sbio) [above = 1.1cm,right = 2cm]  {St Bio};
	\node[main node] (Srich) [above = 1.1cm, right = 6cm]  {St Rich};
	\node[main node] (Suit) [right = 8cm]  {Pl Suit};

    \tikzset{>=latex}

\path[color=black!20!blue,style={->}]
(Sprod) edge node {} (Pprod)
(Sprod) edge node {} (Pbio)
(Sprod) edge node {} (Prich)
(Sbio) edge node {} (Prich)
(Sbio) edge node {} (Pbio)
(Sbio) edge node {} (Pprod)
(Srich) edge node {} (Pprod)
(Srich) edge node {} (Pbio)

(Srich) edge node {} (Prich)

(Pbio) edge node {} (Shade)
(Shade) edge node {} (Prich)
(Shade) edge node {} (Prich)
(Prich) edge[bend left = 20] node {} (Pprod)
(Suit) edge node {} (Prich);

\path[color=black!20!red,style={<->}]
(Suit) edge node {} (Srich)
(Sprod) edge[bend left = 15] node {} (Srich)
(Sbio) edge node {} (Srich)
(Sprod) edge node {} (Sbio);

\end{tikzpicture}
\caption[Data example: BAP respresentation from Grace et al. (2016)]{\label{fig:targetBAP}A subset of the model from \cite{Grace} used in~\cite{BANG}}
\end{subfigure}
\caption{True model}
\end{figure}
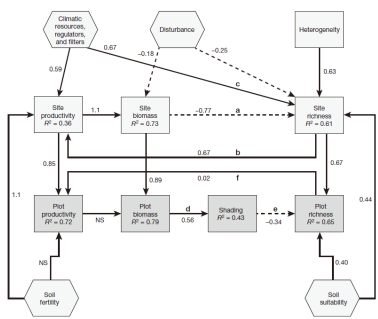
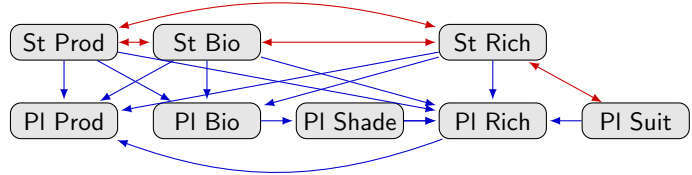

\begin{figure}[H]
    \centering
	\begin{tikzpicture}[->,shorten >=1pt, 
	auto,
	main node/.style={rounded corners,inner sep=0pt,fill=gray!20,draw,font=\sffamily,
	minimum width = 1.5cm, minimum height = .5cm, scale=0.95}]

	\node[main node] (Pprod) [right = 0cm] {Pl Prod};
	\node[main node] (Pbio) [right = 2cm]  {Pl Bio};
	\node[main node] (Shade) [right =4cm]  {Pl Shade};
	\node[main node] (Prich) [right = 6cm]  {Pl Rich};
	\node[main node] (Sprod) [above = 1.1
	cm,right = 0cm]  {St Prod};
	\node[main node] (Sbio) [above = 1.1cm,right = 2cm]  {St Bio};
	\node[main node] (Srich) [above = 1.1cm, right = 6cm]  {St Rich};
	\node[main node] (Suit) [right = 8cm]  {Pl Suit};

    \tikzset{>=latex}
	\path[color=black!20!blue,style={->}]
	(Suit) edge node {} (Srich)
	(Pprod) edge node {} (Srich)
	(Pbio) edge node {} (Shade)
	(Sbio) edge node {} (Pbio)
	(Sbio) edge node {} (Prich)
	(Srich) edge node {} (Prich)
	(Sprod) edge node {} (Pprod)
	(Sprod) edge node {} (Prich)
	(Sprod) edge node {} (Shade)
	(Shade) edge node {} (Prich)
	(Pbio) edge node {} (Pprod)
	;

	\path[color=black!20!red, style={<->}]
	(Suit) edge[bend left = 20] node {} (Pprod)
	(Suit) edge[bend right = 5] node {} (Sbio)
	(Suit) edge node {} (Sprod)
	(Srich) edge[bend right = 20] node {} (Sprod)
	(Srich) edge node {} (Sbio)
	(Sbio) edge node {} (Pprod)
	(Sprod) edge node {} (Sbio)
	;

	\end{tikzpicture}
		\caption[Data example: model discovered by BANG]{\label{fig:discovered}Model discovered  by  BANG~\cite{BANG}.}
\end{figure}
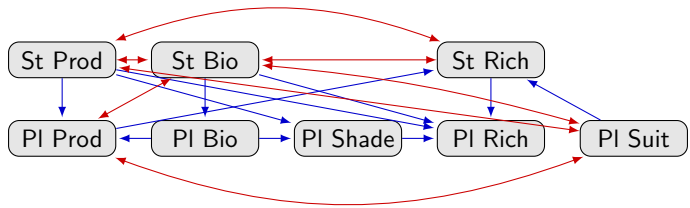

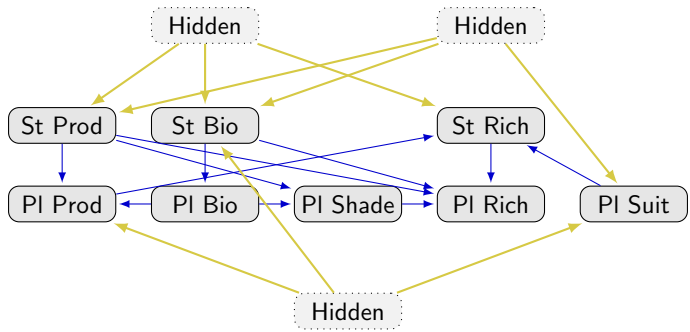
\begin{figure}[H]
    \centering
	\begin{tikzpicture}[->,shorten >=1pt, 
	auto,
	main node/.style={rounded corners,inner sep=0pt,fill=gray!20,draw,font=\sffamily,
	minimum width = 1.5cm, minimum height = .5cm, scale=0.95}]

	\node[main node] (Pprod) [right = 0cm] {Pl Prod};
	\node[main node] (Pbio) [right = 2cm]  {Pl Bio};
	\node[main node] (Shade) [right =4cm]  {Pl Shade};
	\node[main node] (Prich) [right = 6cm]  {Pl Rich};
	\node[main node] (Sprod) [above = 1.1
	cm,right = 0cm]  {St Prod};
	\node[main node] (Sbio) [above = 1.1cm,right = 2cm]  {St Bio};
	\node[main node] (Srich) [above = 1.1cm, right = 6cm]  {St Rich};
	\node[main node] (Suit) [right = 8cm]  {Pl Suit};
	\node[rounded corners,dotted,inner sep=0pt,fill=gray!10,draw,font=\sffamily,
	minimum width = 1.5cm, minimum height = .5cm, scale=0.95] (Hidden) [below=1.5cm, right =4cm] {Hidden};
		\node[rounded corners,dotted,inner sep=0pt,fill=gray!10,draw,font=\sffamily,
	minimum width = 1.5cm, minimum height = .5cm, scale=0.95] (Hidden2) [above=2.5cm, right =6cm] {Hidden};
		\node[rounded corners,dotted,inner sep=0pt,fill=gray!10,draw,font=\sffamily,
	minimum width = 1.5cm, minimum height = .5cm, scale=0.95] (Hidden3) [above=2.5cm, right =2cm] {Hidden};

    \tikzset{>=latex}
	\path[color=black!20!blue,style={->}]
	(Suit) edge node {} (Srich)
	(Pprod) edge node {} (Srich)
	(Pbio) edge node {} (Shade)
	(Sbio) edge node {} (Pbio)
	(Sbio) edge node {} (Prich)
	(Srich) edge node {} (Prich)
	(Sprod) edge node {} (Pprod)
	(Sprod) edge node {} (Prich)
	(Sprod) edge node {} (Shade)
	(Shade) edge node {} (Prich)
	(Pbio) edge node {} (Pprod)
	;

	\path[color=black!20!red, style={<->}]
	%(Suit) edge[bend left = 20] node {} (Pprod)
	%%%(Suit) edge[bend right = 5] node {} (Sbio)
	%(Suit) edge node {} (Sprod)
	%(Srich) edge[bend right = 20] node {} (Sprod)
	%(Srich) edge node {} (Sbio)
	%(Sbio) edge node {} (Pprod)
	%(Sprod) edge node {} (Sbio);
	;
	
	\path[color=black!20!yellow, thick,style={->}]
	(Hidden) edge node {} (Suit)
	(Hidden) edge node {} (Pprod)
	(Hidden) edge node {} (Sbio)
	
	(Hidden2) edge node {} (Suit)
	(Hidden2) edge node {} (Sbio)
	(Hidden2) edge node {} (Sprod)
	
	(Hidden3) edge node {} (Srich)
	(Hidden3) edge node {} (Sbio)
	(Hidden3) edge node {} (Sprod)
	;

`	\end{tikzpicture}
		\caption[Data example: model discovered by MBANG]{\label{fig:discoveredMBANG}Model discovered by MBANG.}
\end{figure}

\section{Further Discussion}\label{sec:5}
In this paper we proposed a high-order cumulant based algorithm for discovering hidden variable structure in non-Gaussian LSEMs.
Note that our Algorithm~\ref{alg:2} can be used  on top of any procedure that recovers a mixed graph $G=(V, \mathcal D, \mathcal B)$ with directed and bidirected edges only and an estimate of the direct effects matrix $B$. We chose the BANG algorithm  as this first step since it applies to a large class of graphs: bow-free acyclic mixed graphs.
%Since our algorithm is built on the BANG algorithm and only uses the estimated coefficient matrix $D$ and the bidirected edges, our MBANG algorithm can also be built on some other algorithms that can also find $D$ and the bidirected edges.

Second, the performance of our algorithm is closely related to the error distribution, and the density and size of the graph. For example, when the errors are unif($-5,\, 5$) and the number of vertices is 7, as the edge number increases from 8 to 20, the correct rate drops from $78\%$ to $1\%$. Also, if the density is medium, as the number of vertices increases from 5 to 7, the correct rate drops from $78\%$ to about $40\%$. Compared to  the uniform distribution, the exponential distribution has a much milder drop in correct rate. For example, when the sample size is 50000, the correct rate for exp(1) drops from $70\%$ to $61\%$ as the graph density increases from sparse to very dense (almost complete). Note, however, that this drop is also present in the output of the BANG algorithm itself.

%The drop of correct rate is mainly caused by the misidentification of the \textit{Bang Procedure}. By applying other algorithms to produce the \textit{Bow-free} mixed graph containing bi-directed edges in the middle step, the algorithm may have correct rate higher for large graph with very high density. 

Last,  while  we proved that Algorithm~\ref{Alg:1} finds the true graph given enough samples, it would be interesting to know the exact amount of samples needed in both BANG and Algorithm~\ref{alg:2}, as well as the nominal levels for the independence tests in  BANG and in our cumulant tests.

\section*{Acknowledgements} We thank Mathias Drton and Samuel Wang for helpful discussions regarding their algorithm~\cite{BANG}. We also thank Jean-Baptiste Seby for a helpful discussion at an early stage  of the project.

ER was supported by an NSERC Discovery Grant (DGECR-2020-00338). YL was supported by a WLIURA in the  summer of  2020.

\bibliographystyle{alpha}

\end{document}